\documentclass[runningheads]{llncs}
\usepackage[english]{babel}
\usepackage{amsmath}
\usepackage{amsfonts}
\usepackage[ruled, vlined, nofillcomment, linesnumbered]{algorithm2e}
\usepackage{subcaption}
\usepackage{xspace}
\usepackage{url}
\usepackage{icomma}
\usepackage{enumitem}
\usepackage{booktabs}
\usepackage{tikz}
\usepackage{pgfplots}
\usepackage{floatrow}
\usepackage[numbers,sort&compress]{natbib}

\usepackage{graphicx}
\newcommand{\set}[1]{\left\{#1\right\}}

\newcommand{\floor}[1]{\left\lfloor#1\right\rfloor}
\newcommand{\fpr}[1]{\mathopen{}\left(#1\right)}

\newcommand{\abs}[1]{{\left|#1\right|}}

\newcommand{\define}{\leftarrow}

\DeclareRobustCommand{\dispfunc}[2]{%
  \ensuremath{%
  \ifthenelse{\equal{#2}{}}%
    {\mathit{#1}}%
    {\mathit{#1}\fpr{#2}}}}

\newcommand{\bigO}[1]{\dispfunc{\mathcal{O}}{#1}}

\newcommand{\leftchild}[1]{\dispfunc{left}{#1}}
\newcommand{\rightchild}[1]{\dispfunc{right}{#1}}

\newcommand{\score}[1]{\dispfunc{s}{#1}}

\newcommand{\poslab}[1]{\dispfunc{p}{#1}}
\newcommand{\neglab}[1]{\dispfunc{n}{#1}}
\newcommand{\accpos}[1]{\dispfunc{accpos}{#1}}
\newcommand{\accneg}[1]{\dispfunc{accneg}{#1}}

\newcommand{\pel}[1]{\dispfunc{prev}{#1}}
\newcommand{\nel}[1]{\dispfunc{next}{#1}}

\newcommand{\nextpos}[1]{\dispfunc{nextpos}{#1}}
\newcommand{\prevpos}[1]{\dispfunc{prevpos}{#1}}
\newcommand{\nextneg}[1]{\dispfunc{nextneg}{#1}}
\newcommand{\prevneg}[1]{\dispfunc{prevneg}{#1}}

\newcommand{\gappos}[1]{\dispfunc{gp}{#1}}
\newcommand{\gapneg}[1]{\dispfunc{gn}{#1}}

\newcommand{\nextcntpos}[1]{\dispfunc{nextcntpos}{#1}}
\newcommand{\prevcntpos}[1]{\dispfunc{prevcntpos}{#1}}
\newcommand{\nextcntneg}[1]{\dispfunc{nextcntneg}{#1}}
\newcommand{\prevcntneg}[1]{\dispfunc{prevcntneg}{#1}}

\newcommand{\headneg}[1]{\dispfunc{hn}{#1}}
\newcommand{\headpos}[1]{\dispfunc{hp}{#1}}

\newcommand{\aucP}{\mathit{C}}

\newcommand{\treeP}{\mathit{TP}}

\newcommand{\Wtilde}[1]{\stackrel{\sim}{\smash{#1}\rule{0pt}{0.5ex}}}

\newcommand{\auc}{\mathit{auc}}
\newcommand{\apauc}{\Wtilde{\mathit{auc}}}

\newcommand{\approxauc}{\textsc{ApproxAUC}\xspace}

\newcommand{\addlist}{\textsc{Add}\xspace}
\newcommand{\removelist}{\textsc{Remove}\xspace}
\newcommand{\cumstats}{\textsc{HeadStats}\xspace}
\newcommand{\maxpos}{\textsc{MaxPos}\xspace}

\newcommand{\addtreepos}{\textsc{AddTreePos}\xspace}
\newcommand{\removetreepos}{\textsc{RemoveTreePos}\xspace}
\newcommand{\addtreeneg}{\textsc{AddTreeNeg}\xspace}
\newcommand{\removetreeneg}{\textsc{RemoveTreeNeg}\xspace}

\newcommand{\addcomppos}{\textsc{AddNext}\xspace}
\newcommand{\compresspos}{\textsc{Compress}\xspace}

\newcommand{\addpos}{\textsc{AddPos}\xspace}
\newcommand{\removepos}{\textsc{RemovePos}\xspace}

\newcommand{\dtname}[1]{\textsl{#1}}

\newcommand{\hepmass}{\dtname{Hepmass}\xspace}
\newcommand{\tvads}{\dtname{Tvads}\xspace}
\newcommand{\miniboone}{\dtname{Miniboone}\xspace}

\SetKwComment{tcpas}{\{}{\}}
\SetCommentSty{textnormal}
\SetArgSty{textnormal}
\SetKw{False}{false}
\SetKw{True}{true}
\SetKw{Null}{null}
\SetKw{None}{none}
\SetKwInOut{Output}{output}
\SetKwInOut{Input}{input}
\SetKw{AND}{and}
\SetKw{OR}{or}
\SetKw{Break}{break}

\pgfdeclarelayer{background}
\pgfdeclarelayer{foreground}
\pgfsetlayers{background,main,foreground}

\definecolor{yafaxiscolor}{rgb}{0.3, 0.3, 0.3}

\definecolor{yafcolor1}{rgb}{0.4, 0.165, 0.553}
\definecolor{yafcolor2}{rgb}{0.949, 0.482, 0.216}
\definecolor{yafcolor3}{rgb}{0.47, 0.549, 0.306}
\definecolor{yafcolor4}{rgb}{0.925, 0.165, 0.224}
\definecolor{yafcolor5}{rgb}{0.141, 0.345, 0.643}
\definecolor{yafcolor6}{rgb}{0.965, 0.933, 0.267}
\definecolor{yafcolor7}{rgb}{0.627, 0.118, 0.165}
\definecolor{yafcolor8}{rgb}{0.878, 0.475, 0.686}

\tikzstyle{exnode} = [inner sep = 1pt]
\tikzstyle{labnode} = [sloped, text = black, font = \scriptsize, inner sep = 1pt]
\tikzstyle{exedge} = [yafcolor5, draw, thick, >=latex, ->]
\tikzstyle{exedge2} = [yafcolor2, draw, thick, >=latex, ->]
\tikzstyle{exedge3} = [yafcolor3, draw, thick, >=latex, ->]

\newlength{\yafaxispad}
\newlength{\yaftlpad}
\newlength{\yaflabelpad}
\newlength{\yafaxiswidth}
\newlength{\yafticklen}
\makeatletter
\def\pgfplots@drawtickgridlines@INSTALLCLIP@onorientedsurf#1{}
\makeatother

\newcommand{\yafdrawxaxis}[2]{
	\pgfplotstransformcoordinatex{#1}\let\xmincoord=\pgfmathresult 
	\pgfplotstransformcoordinatex{#2}\let\xmaxcoord=\pgfmathresult 
	\pgfsetlinewidth{\yafaxiswidth} 
	\pgfsetcolor{yafaxiscolor}
	\pgfpathmoveto{\pgfpointadd{\pgfpointadd{\pgfplotspointrelaxisxy{0}{0}}{\pgfqpointxy{\xmincoord}{0}}}{\pgfqpoint{-0.5\yafaxiswidth}{\yafaxispad}}}
	\pgfpathlineto{\pgfpointadd{\pgfpointadd{\pgfplotspointrelaxisxy{0}{0}}{\pgfqpointxy{\xmaxcoord}{0}}}{\pgfqpoint{0.5\yafaxiswidth}{\yafaxispad}}}
	\pgfusepath{stroke}

}
\newcommand{\yafdrawyaxis}[2]{
	\pgfplotstransformcoordinatey{#1}\let\ymincoord=\pgfmathresult 
	\pgfplotstransformcoordinatey{#2}\let\ymaxcoord=\pgfmathresult 
	\pgfsetlinewidth{\yafaxiswidth} 
	\pgfsetcolor{yafaxiscolor}
	\pgfpathmoveto{\pgfpointadd{\pgfpointadd{\pgfplotspointrelaxisxy{0}{0}}{\pgfqpointxy{0}{\ymincoord}}}{\pgfqpoint{\yafaxispad}{-0.5\yafaxiswidth}}}
	\pgfpathlineto{\pgfpointadd{\pgfpointadd{\pgfplotspointrelaxisxy{0}{0}}{\pgfqpointxy{0}{\ymaxcoord}}}{\pgfqpoint{\yafaxispad}{0.5\yafaxiswidth}}}
	\pgfusepath{stroke}
}

\newcommand{\yafdrawaxis}[4]{\yafdrawxaxis{#1}{#2}\yafdrawyaxis{#3}{#4}}

\pgfplotscreateplotcyclelist{yaf}{%
{yafcolor1,mark options={scale=0.75},mark=o}, 
{yafcolor2,mark options={scale=0.75},mark=square},
{yafcolor3,mark options={scale=0.75},mark=triangle},
{yafcolor4,mark options={scale=0.75},mark=o},
{yafcolor5,mark options={scale=0.75},mark=o},
{yafcolor6,mark options={scale=0.75},mark=o},
{yafcolor7,mark options={scale=0.75},mark=o},
{yafcolor8,mark options={scale=0.75},mark=o}} 

\pgfplotsset{axis y line=left, axis x line=bottom,
	tick align=outside,
	compat = 1.3,
	tickwidth=\yafticklen,
	clip = false,
	every axis title shift = 0pt,
    x axis line style= {-, line width = 0pt, opacity = 0},
    y axis line style= {-, line width = 0pt, opacity = 0},
    x tick style= {line width = \yafaxiswidth, color=yafaxiscolor, yshift = \yafaxispad},
    y tick style= {line width = \yafaxiswidth, color=yafaxiscolor, xshift = \yafaxispad},
    x tick label style = {font=\scriptsize, yshift = \yaftlpad},
    y tick label style = {font=\scriptsize, xshift = \yaftlpad},
    every axis y label/.style = {at = {(ticklabel cs:0.5)}, rotate=90, anchor=center, font=\scriptsize, yshift = -\yaflabelpad},
    every axis x label/.style = {at = {(ticklabel cs:0.5)}, anchor=center, font=\scriptsize, yshift = \yaflabelpad},
    x tick label style = {font=\scriptsize, yshift = 1pt},
    grid = major,
    major grid style  = {dash pattern = on 1pt off 3 pt},
	every axis plot post/.append style= {line width=\yafaxiswidth} ,
	legend cell align = left,
	legend style = {inner sep = 1pt, cells = {font=\scriptsize}},
	legend image code/.code={%
		\draw[mark repeat=2,mark phase=2,#1] 
		plot coordinates { (0cm,0cm) (0.15cm,0cm) (0.3cm,0cm) };%
	} 
}

\begin{document}
\title{Efficient estimation of AUC in a sliding window}
%
%
\author{Nikolaj Tatti}
\authorrunning{N. Tatti}
%
\institute{F-Secure, Helsinki, Finland, 
\email{nikolaj.tatti@gmail.com}}
\maketitle              
\begin{abstract}
In many applications, monitoring area under the ROC curve (AUC) in a sliding
window over a data stream is a natural way of detecting changes in the system.
The drawback is that computing AUC in a sliding window is expensive, especially
if the window size is large and the data flow is significant.

In this paper we propose a scheme for maintaining an approximate
AUC in a sliding window of length $k$. More specifically, we propose an
algorithm that, given $\epsilon$, estimates AUC within $\epsilon / 2$, and can
maintain this estimate in $\bigO{(\log k) / \epsilon}$ time, per update, as the
window slides. This provides a speed-up over the exact computation of AUC,
which requires $\bigO{k}$ time, per update. The speed-up becomes more significant
as the size of the window increases.
Our estimate is based on grouping the data points together, and using these groups
to calculate AUC. The grouping is designed carefully such that
($i$) the groups are small enough, so that the error stays small,
($ii$) the number of groups is small, so that enumerating them is not expensive, and
($iii$) the definition is flexible enough so that we can maintain the groups efficiently.

Our experimental evaluation demonstrates that the average approximation error in
practice is much smaller than the approximation guarantee $\epsilon / 2$, and
that we can achieve significant speed-ups with only a modest sacrifice in
accuracy.

\end{abstract}

\section{Introduction}
Consider monitoring prediction performance in a stream of data points.  That is, we first
receive a data point $d$ without the label, and we predict the missing label
with a score of $s$, after the prediction we receive the true label $\ell$.
We are interested in monitoring how well $s$ predicts $\ell$ as the stream evolves over time.

A good example of such a task is a monitoring system for corporate computers
that detects abnormal behavior based on event logs. Here the positive label
represents an abnormal event that requires a closer inspection, and such a label
can be given, for example, by an expert or triggered automatically. The
produced score can be used for decision making, and can be a specific feature
or a simple statistic, or the result of some classifier, such as logistic
regression.  It is vital to monitor such a system continuously to notice
breakdowns early.  Possible causes may be changes in the underlying
distribution or a system failure, due to the software update.

A natural choice to monitor the predictive power of a real-valued score is the
area under the ROC curve (AUC) in a sliding window over the stream of events as
proposed by~\citet{brzezinski:17:pauc}. Unfortunately, maintaining the exact
AUC requires $\bigO{k}$ time, per new event, where $k$ is the size of the
window. This may be too expensive if $k$ is large and the rate of the events is
significant.

In this paper we propose a technique for estimating AUC efficiently in a
sliding window. Namely, we propose an approximation scheme that has $\epsilon /
2$ approximation error guarantee while having $\bigO{(\log k) / \epsilon}$
update time. That is, the scheme provides a trade-off between the accuracy and
computational complexity. 

Our approach is straightforward. Computing AUC exactly requires sorting data
points and summing over all data points (see Eq.~\ref{eq:auc} for the exact
formula). Maintaining points sorted can be done using binary search trees.
However, estimating the sum requires additional tricks. We approach the problem
by grouping neighboring data points together, that is, treating them as
if the classifier given them the same score.

The key step is to design a grouping such that 3 properties hold at the same time:
($i$) the groups are small enough so that the relative error  is small, more specifically, $\abs{\apauc - \auc} / \auc  \leq \epsilon / 2$,
($ii$) the number of groups is small enough, more specifically, it should be in $\bigO{(\log k) / \epsilon}$, and
($iii$) the definition should be flexible enough so that we can do quick updates whenever points arrive or leave the sliding window.

Roughly speaking, in order to accommodate all 3 demands, we will maintain the
groups with the two following properties:
($i$) the number of positive labels in a group is less than or equal to $(1 + \epsilon)$ than the total number of
positive labels in all the previous groups,
($ii$) the number of positive labels in a group, \emph{and the next group}, is
larger than $(1 + \epsilon)$ than the total number of positive labels in all
the previous groups. The first property will yield the approximation guarantee, while the second property
guarantees that the number of groups remains small. Moreover, these properties are flexible enough
so we can perform update procedures quickly.

The rest of the paper is organized as follows.
We begin by reminding ourselves the definition of AUC in Section~\ref{sec:prel}.
Updating the groups of data points quickly requires several auxiliary structures, which
we introduce in Section~\ref{sec:aux}. We then proceed describing AUC estimation in Section~\ref{sec:auc}.
The related work is given in Section~\ref{sec:related}. In Section~\ref{sec:exp}, we demonstrate
that the relative error in practice is much smaller than the guaranteed bound, as well as, study
the trade-off between the error and the computational cost. Finally, we conclude the paper with discussion in Section~\ref{sec:conclusions}.

\section{Preliminaries}\label{sec:prel}

We start with the definition of AUC, and provide a formula for computing it.

Assume that we are given a set of $k$ pairs $W = (s_i, \ell_i)_i^k$, where $\ell_i$ is
the true label of the $i$th instance, $\ell_i = 0, 1$, and $s_i$ is score
produced by the classification algorithm. The larger $s_i$, the more we believe
that $\ell_i$ should be $0$.\!\footnote{We chose this direction due to the notational convenience.}

In order to predict a label, we need a threshold
$\sigma$, and predict that $\ell_i = 0$ if $s_i \geq \sigma$, and $\ell_i = 1$ otherwise.
The ROC curve is obtained by varying $\sigma$ and plotting true positive rate
as a function of false positive rate. AUC is the area under the ROC curve.
To compute AUC, we can use the following formula. Let
\[
\neglab{s} = \abs{\set{i \mid s_i = s, \ell_i = 0}} \quad\text{and}\quad
\poslab{s} = \abs{\set{i \mid s_i = s, \ell_i = 1}}
\]
be the counts of labels with a score of $s$. Define also $\headpos{s} = \sum_{t
< s} \poslab{t}$. Then,
\begin{equation}
\label{eq:auc}
	\auc = \frac{1}{A}\sum_{s} (\headpos{s} + \frac{1}{2}\poslab{s})\neglab{s},
\end{equation}
where $A = \abs{\set{i \mid \ell_i = 0}}\abs{\set{i \mid \ell_i = 1}}$ is the
normalization factor.  Eq.~\ref{eq:auc} can be computed in $\bigO{k \log k +
k}$ time by first sorting $W$, computing $\headpos{}$, and enumerating over the sum of Eq.~\ref{eq:auc}.

In a streaming setting, $W$ is a sliding window, and our goal is to compute
AUC as $W$ slides over a stream of predictions and labels.

\section{Supporting data structures for estimating AUC}\label{sec:aux}

In this section we introduce supporting data structures that are needed to compute AUC in
a streaming setting. Additional structures and the actual logic for computing AUC
are given in the next section. We begin by describing the data structures, then follow
with introducing the needed query operations, and finally finish with explaining the update procedures.

\subsection{Data structures}

Assume that we have a sequence of pairs $W = (s_i, \ell_i)_{i = 1}^k$, where $s_i$
is the score produced by the classifier, and $\ell_i \in \set{0, 1}$ is the
true label.

We store $W$ in a red-black tree
$T$ sorted by the scores $s_i$. Let $v \in T$ be a node in $T$.  We will denote
the corresponding score of $v$ by $\score{v}$.  We store and maintain the
following information:

\begin{itemize}
\item Counter $\poslab{v} = \abs{\set{i \mid s_i = \score{v}, \ell_i = 1}}$, number of pairs in $W$ with a score $\score{v}$ and a positive label.
\item Counter $\neglab{v} = \abs{\set{i \mid s_i = \score{v}, \ell_i = 0}}$, number of pairs in $W$ with a score $\score{v}$ and a negative label.
\item Counter $\accpos{v}$, the total sum of $\poslab{w}$, where $w$ ranges over all
descendant nodes of $v$ in $T$, including $v$ itself.
\item Counter $\accneg{v}$, the total sum of $\neglab{w}$, where $w$ ranges over all
descendant nodes of $v$ in $T$, including $v$ itself.
\end{itemize}

For simplicity, we will add two sentinel nodes to $T$. The first node will
have a score of $-\infty$ and the second node has a score $\infty$. We will
assume that the actual entries will never achieve these values. Both sentinel nodes 
have 0 positive labels and 0 negative labels.

Note that if the scores $s_i$ are unique, then we have either $\poslab{v} = 1$,
$\neglab{v} = 0$, or $\poslab{v} = 0$, $\neglab{v} = 1$. However, if there are
duplicate scores, then we may have any integer combinations.

In addition to red-black trees, we need to maintain several linked lists,
for which we will now introduce the notation.
Assume that we are given a subset $U$ of nodes in $T$. We would like to
maintain $U$ in a linked list $L$, sorted by the score. For that we will need two
pointers for each node $u \in U$, namely, $\nel{u; L}$ indicating the next
node in $L$, and $\pel{u; L}$ indicating the previous node in $L$.
Let $u \in U$ and assume that $v = \nel{u; L}$ exists. Let 
\[
	B = \set{w \in T \mid \score{u} \leq \score{w} < \score{v}}
\]
be the set of nodes in $T$ between $u$ and $v$. We define
\[
	\gappos{u; L} = \sum_{w \in B} \poslab{w} \quad\text{and}\quad
	\gapneg{u; L} = \sum_{w \in B} \neglab{w} 
\]
to be the total sums of the labels in the gap $B$. We will refer to $L$ as \emph{weighted linked list}.
Note that deleting an element from $L$ and maintaining the gap counters can be done in constant
time. We will refer to the deletion algorithm by $\removelist(L, v)$.
Moreover, adding a new element, say $v$, to $L$ after $u$ can be also done in constant time,
if we already know the total sums of labels, say $p$ and $n$, between $u$ and $v$. 
We will refer to the insertion algorithm by $\addlist(L, u, v, p, n)$.

We say that the node $v \in T$ is \emph{positive}, if $\poslab{v} > 0$.
Similarly, we say that the node $v$ is \emph{negative}, if $\neglab{v} > 0$.
Note that $v$ can be both negative and positive.

We maintain all positive nodes in a weighted linked list, which we will refer
as $P$.
Finally, we also store all positive nodes in its own dedicated red-black tree,
denoted by $\treeP$.
For simplicity, we also store the sentinel nodes of $T$ in $P$ and $\treeP$ as the first and the last nodes.

\subsection{Query procedures}

The first query that we need is $\maxpos(s)$, returning the
\emph{positive} node $v$ with the largest score such that $\score{v} \leq s$.
This can be done in $\bigO{\log k}$ time using $\treeP$, where $k$ is the
number of elements in the window.


Maintaining $\accpos{v}$ and $\accneg{v}$ allows us to query a cumulative
sums of counts. Specifically, given a score $s$, we are interested in
\begin{equation}
	\headpos{v} = \sum_{v \in T \mid \score{v} < s} \poslab{v} \quad\text{and}\quad
	\headneg{v} = \sum_{v \in T \mid \score{v} < s} \neglab{v}\quad.
\end{equation}
We can compute both of these sums with $\cumstats(s)$, given in Algorithm~\ref{alg:cumstats}.

\begin{algorithm}
\caption{$\cumstats(s)$, computes the cumulative counts of labels,
$\headpos{v}$
and
$\headneg{v}$. Assumes that a node in $T$ with a score $s$ exists.}
\label{alg:cumstats}

$hp \define 0$;
$hn \define 0$;

$v \define $ root of $T$\;

\While {\True} {
	\uIf {$\score{v} < s$} {
		$v \define \leftchild{v}$\;
	}
	\Else {
		\If {$\leftchild{v}$} {
			$hp \define hp + \accpos{\leftchild{v}}$\;
			$hn \define hn + \accneg{\leftchild{v}}$\;
		}
		\uIf {$\score{v} = s$} {
			\Return $hp, hn$\;
		}
		\Else {
			$hp \define hp + \poslab{v}$\;
			$hn \define hn + \neglab{v}$\;
			$v \define \rightchild{v}$\;
		}
	}
}
\end{algorithm}

The algorithm assumes that there is a node in $T$ containing $s$, and proceeds to find it;
during the search whenever we go the right branch we add the accumulative sums from the left branch.
We omit the trivial proof of correctness.
Since the tree is balanced, the running time of $\cumstats(s)$ is $\bigO{\log k}$, where $k$ 
is the number of entries in the window.

\subsection{Update procedures}
We now continue to the maintenance procedures as we slide the
window. This comes down to two procedures: ($i$) removing an entry from the window
and ($ii$) adding an entry to the window.


We will
first describe removing an entry with a positive label and a score $s$.
First we will find the node, say $v$, with the score $s$, and reduce the
counter
$\poslab{v}$ by 1. We will need to update the $\accpos{}$ counters. However, we
only need to do it for the ancestors of $v$, and there are only $\bigO{\log k}$ of
them, where $k$ is the number of entries in the window, since $T$ is balanced. We also reduce $\gappos{v; P}$ by 1.
In the process, $v$ may become non-positive, and we need to delete it from $\treeP$
as well as from $P$.

Finally, if $\poslab{v} = \neglab{v} = 0$, we need to delete the node from $T$.
This may result in rebalancing of the tree, and during the balancing we need to make
sure that the counters $\accpos{}$ and $\accneg{}$ are properly updated. Luckily,
the red-black tree balancing is based on left and right rotations. During these
rotations it is easy to maintain the counters without additional costs.

We will refer to this procedure as $\removetreepos(s)$ and the pseudo-code is given in Algorithm~\ref{alg:removetreepos}.
$\removetreepos(s)$ runs in $\bigO{\log k}$ time.

\begin{algorithm}
\caption{$\removetreepos(s, T, \treeP, P)$, removes an entry to $T$ with a positive label and a score $s$.}
\label{alg:removetreepos}
$v \define $ node with score $s$ in $T$\;
update $\poslab{v}$, $\gappos{v; P}$, and $\accpos{}$ counters of the ancestors of $v$\;
\lIf {$\poslab{v} = 0$} {
	remove $v$ from the linked list $P$ and the search tree $\treeP$%
}
\lIf{$\poslab{v} = \neglab{v} = 0$} {remove $v$ from $T$}
\end{algorithm}

Deleting an entry with a negative label and a score $s$ is simpler.
First, we find the node, say $v$, with the score $s$, and reduce the
$\neglab{v}$ counter by 1. If needed, we delete $v$ from $T$.
Finally we use $\maxpos(s)$ to find $u$, the largest positive node with $\score{u} \leq u$,
and reduce $\gapneg{u; P}$ by 1.
The procedure, referred as \removetreeneg, runs in  $\bigO{\log k}$ time.

Next, we will describe the addition of a positive entry with a score $s$.
First, we will add the entry $s$ to $T$, possibly creating a new node in the process.
Let $v$ be the node in $T$ with the score $s$.

If $v$ is a new node, then we need to add it to the weighted linked list $P$.
First, we find the node, say $w = \maxpos(s)$, after which $v$ is supposed to be added. We need to compute
the new gap counter $\gapneg{v; P}$. By definition, this value is equal to the total
count of negative labels of nodes between $w$ and $v$, including $w$. 
Thus, this new gap counter is equal to $\headneg{w} - \headneg{v}$.
Both counters can be obtained using $\cumstats{}$ in $\bigO{ \log k}$ time.


We will refer to this procedure as $\addtreepos(s)$, and the pseudo-code is given in Algorithm~\ref{alg:addtreepos}.
$\addtreepos(s)$ runs in $\bigO{\log k}$ time.

\begin{algorithm}
\caption{$\addtreepos(s)$, adds an entry to $T$ with a positive label and a score $s$.}
\label{alg:addtreepos}

$w \define \maxpos(s)$\;
add $s$ to $T$ (possibly creating new node), and update $\accpos{}$ and $\poslab{}$ counters\;
$v \define $ node with score $s$ in $T$\;
\If {$w \neq v$} {
	add $v$ to $\treeP$\;	
	$p_1, n_1 \define \cumstats(\score{w})$\;
	$p_2, n_2 \define \cumstats(\score{v})$\;
	$\addlist(P, w, v, 1, n_2 - n_1)$ \;
}
\Return $v$\;
\end{algorithm}

Adding an entry with negative label and a score $s$ is simpler.
First, we will add the entry $s$ to $T$, possibly creating a new node in the process.
Let $v$ be the node in $T$ with a score $s$.
Then, we use $\maxpos(s)$ to find $u$, the largest positive node with $\score{u} \leq u$,
and increase $\gapneg{u}$ by 1. The procedure, referred as \addtreeneg, runs in  $\bigO{\log k}$ time.

\section{Estimating AUC efficiently}\label{sec:auc}

In order to approximate AUC, we will use Eq.~\ref{eq:auc} as a basis. However,
instead of enumerating over every node we will enumerate only over some
selected nodes. The key is how to select the nodes such that we will obtain
the approximation guarantee while keeping the number of nodes small.

We will maintain a weighted linked list $\aucP$. Given $\alpha > 1$, we say that
$\aucP$ is $\alpha$-compressed, if for every two consecutive nodes in $\aucP$, say $v$ and $w$,
it holds that 
\begin{equation}
\label{eq:alpha1}
	\headpos{w} \leq \alpha (\headpos{v} + \poslab{v}),
\end{equation}
and if $u = \nel{w; \aucP}$ exists, then 
\begin{equation}
\label{eq:alpha2}
	\headpos{u} > \alpha (\headpos{v} + \poslab{v})\quad.
\end{equation}

Eq.~\ref{eq:alpha1} will yield the approximation guarantee, while
the Eq.~\ref{eq:alpha2} will guarantee the running time.

\subsection{Computing approximate AUC}

Our next step is to show how we can approximate AUC using a compressed list $L$
in $\bigO{L}$ time. The idea is as follows. Let $B$ be the set of nodes between
two consecutive nodes $v$ and $w$ in $L$. Normally, we would have to go over
each individual node in $B$ when computing AUC. Instead, we will group $B$
to a \emph{single} node. We will use 
the total number of positive labels in $B$, that is, $\gappos{v; L} - \poslab{v}$, for the number of positive labels for this node.
Similarly, we will use $\gapneg{v; L} - \neglab{v}$ for the negative labels.
The pseudo-code for the algorithm is given in Algorithm~\ref{alg:approxauc}.

\begin{algorithm}
\caption{$\approxauc(L)$ computes approximate AUC using a weighted linked list.}
\label{alg:approxauc}
$hp \define 0$; $a \define 0$\;
\While {$v \in L$} {
	$p \define \poslab{v}$; $n \define \neglab{v}$\;
	$a \define a + (hp + p / 2)n$\;
	$hp \define hp + p$\;
	$p \define \gappos{v; L} - \poslab{v}$; $n \define \gapneg{v; L} - \neglab{v}$\;
	$a \define a + (hp + p / 2)n$\;
	$hp \define hp + p$\;
}
$A \define (\text{total number of positive labels})\times(\text{total number of negative labels})$\;
\Return $a / A$\;
\end{algorithm}

Let us first establish that \approxauc produces an accurate estimate.

\begin{proposition}
\label{prop:approx}
Let $L$ be $(1 + \epsilon)$-compressed list constructed from the search tree $T$.
Let $\apauc = \approxauc(L)$ be an approximate AUC, and let $\auc$ be the correct AUC.
Then $\abs{\apauc - \auc} \leq \epsilon \auc / 2$.
\end{proposition}

\begin{proof}
Let $A$ be as defined in \approxauc.
Let $v \in T$ be a node, and let $u$ be the node in $L$ with the largest score such that $\score{u} < \score{v}$.
Let $w = \nel{u; L}$ be the next node. Define 
\[
	c_v = \frac{1}{2}(\headpos{u} + \poslab{u} + \headpos{w}) \quad.
\]
Then, \approxauc returns
\begin{equation}
\label{eq:approxauc}
	\apauc = \frac{1}{A}\sum_{v \in L} (\headpos{v} + \frac{1}{2}\poslab{v})\neglab{v}
	+ \sum_{v \in T \setminus L} c_v \neglab{v} \quad.
\end{equation}
We will argue the approximation guarantee by comparing the terms in Eq.~\ref{eq:auc} and Eq.~\ref{eq:approxauc}.  
Let $v$ be a node in $L$. Then the corresponding term can be found in sums of both equations.

Let $v \in T \setminus L$, and write $b = \headpos{v} + \frac{1}{2}\poslab{v}$.
Let $u$ be the node in $L$ with the largest score
such that $\score{u} \leq \score{v}$.  Let $w = \nel{u; L}$ be the next node.
By definition, we have $\headpos{u} + \poslab{u} \leq b \leq \headpos{w}$.
Since $c_v$ is the average of the lower bound and the upper bound, we have
\[
	\abs{b - c_v} \leq \frac{1}{2}(\headpos{w} - \headpos{u} - \poslab{u})  \leq \frac{\epsilon}{2}  (\headpos{u} + \poslab{u}) \leq \frac{\epsilon b}{2},
\]
where the second inequality follows since $L$ is $(1 + \epsilon)$-compressed.

We have shown that the approximation holds for individual terms. Consequently, it holds for the summands $\apauc$ and $\auc$, completing the proof.
\qed
\end{proof}

Two remarks are in order. First, since AUC is always smaller than 1, Proposition~\ref{prop:approx} implies
that the approximation is also absolute, $\abs{\apauc - \auc} \leq \epsilon / 2$. The relative approximation is more accurate
if AUC is small. However, if AUC is close to 1, it may make sense to reverse the approximation guarantee, 
that is, modify the algorithm such that we have a guarantee of $\abs{\apauc - \auc} \leq (1 - \auc)\epsilon / 2$.
This can be done by flipping the labels, and using $1 - \approxauc(\aucP)$ as the estimate.

\approxauc runs in $\bigO{\abs{L}}$ time. Next we establish that $\abs{L}$ is small.

\begin{proposition}
Let $L$ be $(1 + \epsilon)$-compressed list.
Then
$\abs{L} \in \bigO{ \frac{\log k }{ \epsilon}}$,
where $k$ is the number of entries in the sliding window.
\end{proposition}

\begin{proof}
Write $L = u_0, \ldots, u_m$.
Since $L$ is $(1 + \epsilon)$-compressed,
$\headpos{u_2} \geq 1$ and $\headpos{u_{i + 2}} > (1 + \epsilon) \headpos{u_i}$.
Since $\headpos{u_m} \leq k$, we have
$(1 + \epsilon)^{\floor{m / 2}  - 1} \leq k$. Solving for $m$ leads to
$m \in \bigO{ \frac{\log k}{\log{1 + \epsilon}}} \subseteq \bigO{ \frac{\log k}{ \epsilon}}$.
\qed
\end{proof}

\subsection{Updating the data structures}

Our final step is to describe procedures for maintaining $\aucP$ as the data
window slides.
In the previous section, we already described how to update the search trees
$T$ and $\treeP$ as well as the weighed linked list $P$. Our next step is to
make sure that the weighted linked list $\aucP$ stays $\alpha$-compressed.

We will need two utility routines. The first routine, $\addcomppos$, given in
Algorithm~\ref{alg:addnext}, takes as input a node included in both $P$ and
$\aucP$, and adds to $\aucP$ the next node in $P$.
This procedure will be used
extensively to add extra nodes to $\aucP$ so that Eq.~\ref{eq:alpha1} is satisfied.

\begin{algorithm}
\caption{$\addcomppos(v, L, P)$, adds the following node of $v$ in $P$ to $L$. Here $P$ is the weighted linked list
of all positive labels, and $v$ is a node in $P$ and $L$.}
\label{alg:addnext}
$w \define \nel{v, P}$\;
$p \define \gappos{v, P}$;
$n \define \gapneg{v, P}$\;
\lIf{$w \notin L$}{$\addlist(L, v, w, p, n)$}  
\end{algorithm}

Next, we demonstrate how \addcomppos enforces Eq.~\ref{eq:alpha1}.

\begin{lemma}
\label{lem:addcorrpos}
Assume that a linked list $L$ satisfies Eq.~\ref{eq:alpha1} for consecutive positive nodes $v$ and $w$.
Add or remove a single positive entry with a score $s$, and assume that $v$ and $w$ are still positive.
Let $u$ be the next positive node from $v$ in $P$, and
let $L'$ be the list obtained from $L$ by adding a \emph{positive} node $u$.
Then Eq.~\ref{eq:alpha1} holds for $L'$ for the nodes $v$ and $u$ as well as for the nodes $u$ and $w$.
\end{lemma}

\begin{proof}
Let us write $c_x = \headpos{x}$ before modifyng $T$, and $c_x' = \headpos{x}$ after the modification.
Similarly, write $b_x = \poslab{x}$ before the modification, and $b_x' = \poslab{x}$ after the modification.

Since $u$ is the next positive node of $v$, we have
	$c_u' = c_v' + b_v' \leq \alpha(c_v' + b_v')$,
proving the case of $v$ and $u$.

If $s \geq \score{w}$, then $c_w' = c_w \leq \alpha(c_v + b_v) = \alpha c_u = \alpha c_u' \leq \alpha (c_u' + b_u')$.

If we are adding $s$ and $s < \score{w}$, then
\[
	c_w' = c_w + 1 \leq \alpha(c_v + b_v + 1) \leq \alpha(c_v' + b_v' + 1) = \alpha(c_u' + 1) \leq \alpha(c_u' + b_u'),
\]
where the last inequality holds since $u$ is a positive node.

If we are removing $s$ and $s < \score{w}$, then $c_v + b_v - 1 \leq c_v' + b_v'$, and so
\[
	c_w' \leq c_w \leq \alpha(c_v + b_v) \leq \alpha(c_v' + b_v' + 1) = \alpha(c_u' + 1) \leq \alpha(c_u' + b_u').
\]
This proves the case for $u$ and $w$, and completes the proof.\qed
\end{proof}

Note that the execution of \addcomppos is done in constant time, the key step for
this being able to obtain $\gappos{v, P} = \poslab{v}$ and $\gapneg{v, P}$ in constant time. This is the main reason why we maintain
$P$.

While the first utility algorithm adds new entries to $\aucP$,
our second utility algorithm, \compresspos, given in Algorithm~\ref{alg:compress} tries to delete as many entries as possible. It assumes that the input list $\aucP$ already
satisfies Eq.~\ref{eq:alpha1}, and searches for violations of Eq.~\ref{eq:alpha2}. Whenever such
violation is found, the algorithm proceeds deleting the middle node. Note that deleting this node will not violate Eq.~\ref{eq:alpha1}.
Consequently, upon termination, the resulted linked list will be $\alpha$-compressed.
The computational complexity of $\compresspos(\aucP, \alpha)$ is $\bigO{\abs{\aucP}}$.

\begin{algorithm}
\caption{$\compresspos(L, \alpha)$, forces a weighted linked list $L$ that satisfies Equation~\ref{eq:alpha1} to also
satisfy Equation~\ref{eq:alpha2}, making $L$ $\alpha$-compressed.}
\label{alg:compress}
$v \define $ first element in $L$\;
$c \define 0$\;

\While {$\nel{\nel{v; L}; L}$ exists} {
	$w \define \nel{v; L}$\;
	\uIf {$c + \gappos{v; L} + \gappos{w; L} \leq \alpha (c + \poslab{v})$} {
		delete $w$ from $L$\;
	}
	\Else {
		$c \define c + \gappos{v; L}$\;
		$v \define w$\;
	}
}
\end{algorithm}

Next, we describe the update steps.
We will start with the easier ones:

\paragraph{Adding negative entry:}
Given a negative entry with a score $s$, we first invoke \addtreeneg.
Then we
search $u \in \aucP$ with the largest score such that $\score{u} \leq s$.
Once this entry is found, we increase $\gapneg{u; \aucP}$ by 1.

\paragraph{Removing negative entry:}
Given a negative entry with a score $s$, we first invoke \removetreeneg. Then we
search $u \in \aucP$ with the largest score such that $\score{u} \leq s$.
Once this entry is found, we decrease $\gapneg{u; \aucP}$ by 1.

Since the positive labels are not modified, $\aucP$ remains
$\alpha$-compressed, so there is no need for modifying $\aucP$.
The running
time for both routines is $\bigO{\log k + \frac{\log k}{\epsilon}}$.

Let us now consider more complex cases:

\paragraph{Adding positive entry:}
Given a positive entry with a score $s$, we first invoke \addtreepos. 
Then we search $u \in \aucP$ with the largest score such that $\score{u} \leq s$.
Once this entry is found, we increase $\gappos{u; \aucP}$ by 1.
By doing so, we may have violated Eq.~\ref{eq:alpha1} for $u$.
Lemma~\ref{lem:addcorrpos} states that we can correct the problem
by adding the next positive node for each violation.
However, a closer inspection of the proof shows that there can be
only one violation, namely $u$. Consequently, we check if Eq.~\ref{eq:alpha1} holds for $u$,
and if it fails, we add the next positive node by invoking  $\addcomppos(u, \aucP, P)$.
Finally, we call $\compresspos(\aucP, \alpha)$ to force Eq.~\ref{eq:alpha2}; ensuring that $\aucP$
is $\alpha$-compressed.
The pseudo-code for \addpos is given in Algorithm~\ref{alg:addpos}.

\begin{algorithm}
\caption{$\addpos(s, \alpha; T, \treeP, P, \aucP)$, adds an entry with a positive label and a score $s$, updates the tree structures $T$ and $\treeP$ and
the weighted linked lists $P$ and $\aucP$.}
\label{alg:addpos}
$v \define \addtreepos(s, T, \treeP, P)$\;
$u \define \arg \max \set{\score{w} \mid w \in \aucP, \score{w} \leq s}$ \;
$\gappos{u; C} \define \gappos{u; C} + 1$\;
$c \define \sum_{w \in \aucP \mid \score{w} < \score{u}} \gappos{w; C} $\tcpas*{$c = \headpos{u}$}
\lIf {$c + \gappos{u; C} > \alpha (c + \poslab{v})$} {
	$\addcomppos(u, \aucP, P)$
}
$\compresspos(\aucP, \alpha)$\;
\end{algorithm}

\paragraph{Removing positive entry:}
Assume that we are given a positive entry with a score $s$.
First we search $u \in \aucP$ with the largest score such that $\score{u} \leq s$.
Once this entry is found, we decrease $\gappos{u; \aucP}$ by 1.
If $u$ is no longer positive, we add the next positive entry to $\aucP$ and
delete $u$ from $\aucP$. The reason for this is explained later.
We proceed by deleting the entry from the search trees with \removetreepos.

Next we make sure that Eq.~\ref{eq:alpha1} holds for every consecutive
nodes $v$ and $w$. There are two possible cases: ($i$) $v$ and $w$ were
consecutive nodes in $\aucP$ before the deletion, or ($ii$) $u$ was deleted from
$\aucP$, and $w$ was the next positive node before the deletion.
In the first case,
Lemma~\ref{lem:addcorrpos} guarantees that using \addcomppos forces Eq.~\ref{eq:alpha1}.
In the second case, note that $\headpos{w}$ \emph{after} the deletion is equal to
$\headpos{u}$ \emph{before} the deletion of $u$. This implies that since
Eq.~\ref{eq:alpha1} held for $v$ and $u$ before the deletion, Eq.~\ref{eq:alpha1} holds for $v$ and $w$ after the deletion.
Finally, we enforce
Eq.~\ref{eq:alpha2} with \compresspos.
The pseudo-code for \removepos is given in Algorithm~\ref{alg:removepos}.

\begin{algorithm}
\caption{$\removepos(s, \alpha; T, \treeP, P, \aucP)$, removes an entry with a positive label and a score $s$, updates the tree structures $T$ and $\treeP$ and
the weighted linked lists $P$ and $\aucP$.}
\label{alg:removepos}

$u \define \arg \max \set{\score{w} \mid w \in \aucP, \score{w} \leq s}$\;
$\gappos{u} \define \gappos{u} - 1$\;

\uIf {$u \in \aucP$ \AND $\poslab{u} = 1$} {
	$\addcomppos(u, \aucP, P)$\;
	$\removelist(\aucP, u)$\;
}

$\removetreepos(s, T, \treeP, P)$\;

$v \define $ first element in $\aucP$\;
$c \define 0$\;
\While {$\nel{v; \aucP}$ exists} {
	$w \define \nel{v; \aucP}$\;
	$x \define \gappos{v; \aucP}$\;
	\lIf {$c + x > \alpha (c + \poslab{v})$} {
		$\addcomppos(v, \aucP, P)$%
	}
	$c \define c + x$\;
	$v \define w$\;
}

$\compresspos(\aucP, \alpha)$\;
\end{algorithm}

In both routines, modifying the search trees is done in $\bigO{\log k}$ time,
while modifying $\aucP$ is done in $\bigO{\abs{\aucP}} \subseteq \bigO{\frac{\log k}{\epsilon}}$
time.

\section{Related work}\label{sec:related}

The closest related work is a study by~\citet{bouckaert:06:auc}, where the
author divided the ROC curve area into bins, allowing only to maintain the
counters for individual bins.  However, the number of the bins as well as the bins
were static, and no direct approximation guarantees were provided. 

Using AUC in a streaming setting was proposed in a paper by
\citet{brzezinski:17:pauc}. Here the authors use red-black tree, similar to
$T$, to maintain the order of the data points in a sliding window, but they recompute
the AUC from scratch every time, leading to a update time of $\bigO{k + \log k}$.
In fact, our approach is essentially equivalent to their approach if we set $\epsilon = 0$.

Note that using AUC is useful if we do not have a threshold to binarize the
score. If we do have such a threshold, then we can easily maintain a confusion
matrix, which allows us to compute many metrics, such as,
accuracy, recall, $F1$-measure~\cite{gama2013evaluating,gama2010knowledge}, and
Kappa-statistic~\citep{bifet2010sentiment,vzliobaite2015evaluation}. However,
determining such a threshold may be extremely difficult since it depends on the
misclassification costs. Selecting such costs may come down to a(n educated)
guess. 

We based our AUC calculation on a sliding window, that is, we abruptly forget
the data points after certain period of time. The other option is to gradually
forget the data points, for example using an exponential decay (see a survey
by~\citet{gama2014survey} for such examples). There are currently no
methodology for efficiently estimating AUC under exponential decay, and this is
a promising future line of work.

In a related line of work, training a classifier by optimizing AUC in a static
setting has been proposed by~\citet{ataman2006learning,ferri2002learning,brefeld2005auc,herschtal2004optimising}. Here, AUC is used as an optimization criterion,
and needs to be recomputed from scratch in $\bigO{\abs{D} \log \abs{D}}$ time.
Naturally, this may be too expensive for large databases.
\citet{calders:07:auc} estimated AUC as a continuous function. This allowed to
view AUC as a smooth function, and optimize the parameters of the underlying
classifier efficiently using gradient descent techniques. While the underlying
problem is the same as ours, that is, computing AUC from scratch is expensive, the maintenance
procedures make problems orthogonal: in our settings we are required to do updates
when a single data point leaves or enters to our window, whereas here AUC needs
to be recomputed since the scores (and the order) for all existing data points have changed.
However, it may be possible and fruitful to use similar tricks in order to
speed-up the AUC calculation when optimizing classifiers. We leave this as a future line
of work.

\citet{hand:09:alternative} proposed a fascinating alternative for AUC. Namely,
the author views AUC as the optimal classification loss averaged (with weights)
over misclassification cost ratio.  He then argues that AUC evaluates
incoherently, namely the cost ratio weights depend on the ROC curve, and then
he proposes a different coherent alternative. The computation of proposed
metric, though more complex, shares some similarity with AUC, and it may be
possible to use similar techniques as in this paper to approximate this
measure efficiently in a stream.

\section{Experimental evaluation}\label{sec:exp}
In this section we present our experimental evaluation. We have two goals: to
demonstrate the relative error in practice as a function of the guaranteed error,
and to demonstrate the trade-off between the computational cost and the
error.

We implemented calculation of AUC using C++, and conducted the experiments
using Macbook Air (1.6 GHz Intel Core i5 / 8 GB Memory).\!\footnote{See \url{https://bitbucket.org/orlyanalytics/streamauc} for the implementation.} As a classifier we
used Python's scikit implementation of logistic regression. Computing AUC
was done in a separate job from training the classifier as well as scoring
new data points; the reported running times measure only the computation of AUC
over the whole test data.

We used 3 UCI datasets\footnote{\url{https://archive.ics.uci.edu/}} for our experiments, see Table~\ref{tab:basic}:
($i$) \hepmass, a dataset containing features from simulated particle collisions, split in training and test datasets.
We used the \hepmass-1000 variant. Due to the memory restrictions of Python, 
we only used a sample of $500\,000$ data points from training data. We used the whole test dataset.
($ii$) \miniboone: a data used to distinguish electron neutrinos from muon neutrinos.
Since the original data has data points ordered by label, we permuted the dataset
and split it to training and test data.
($iii$) \tvads: a data containing features for identifying commercials from TV news channels.
We used BBC and CNN channels as training data, and the remaining channels as test data.

\begin{table}

\caption{Basic characteristics of the benchmark datasets.}
\label{tab:basic}

\begin{tabular*}{\textwidth}{@{\extracolsep{\fill}}lrr}
\toprule
Dataset & size of training dataset & size of test dataset\\
\midrule
\hepmass   & $500\,000$ & $3\,500\,000$\\
\miniboone &  $30\,064$ &    $100\,000$\\
\tvads     &  $40\,265$ &     $89\,420$ \\
\bottomrule
\end{tabular*}
\end{table}

\newlength{\imgwidth}
\setlength{\imgwidth}{2.6cm}

\newlength{\imgheight}
\setlength{\imgheight}{2.3cm}

\textbf{Actual error vs. guarantee}:
Proposition~\ref{prop:approx} states that the error cannot be more than $\epsilon / 2$.
First, we test the actual relative error, that is,  $\abs{\apauc - \auc} / \auc$
as a function of $\epsilon$.  Here we set the sliding window size to be $1000$.

The top row of
Figure~\ref{fig:relative} shows the relative error, averaged over all sliding windows,
and the bottom row of
Figure~\ref{fig:relative} shows the relative error, maximized over all sliding windows.
From the results we see that both maximum and average error are smaller than the guaranteed.
Especially, the average error is typically smaller of several orders than the theoretical guarantee.
As expected, both errors tend to increase as $\epsilon$ increases.

\begin{figure}[ht!]
\begin{tabular}{rrr}
\subcaptionbox{\hepmass\label{fig:foo}}{
\begin{tikzpicture}
\begin{axis}[xlabel={$\epsilon$}, ylabel= {average relative error},
    width = \imgwidth,
	height = \imgheight,
    cycle list name=yaf,
	tick scale binop=\times,
	scale only axis,
	no markers
    ]
\addplot+[yafcolor5] table[x index = 0, y index = 3, header = false] {results/hepmass.scores.res};
\pgfplotsextra{\yafdrawaxis{0}{0.9}{0}{0.012}}
\end{axis}
\end{tikzpicture}} &
\subcaptionbox{\tvads\label{fig:foo}}{
\begin{tikzpicture}
\begin{axis}[xlabel={$\epsilon$}, ylabel= {average relative error},
    width = \imgwidth,
	height = \imgheight,
    cycle list name=yaf,
	tick scale binop=\times,
	scale only axis,
	no markers
    ]
\addplot+[yafcolor5] table[x index = 0, y index = 3, header = false] {results/tv.scores.res};
\pgfplotsextra{\yafdrawaxis{0}{0.9}{0}{0.0042}}
\end{axis}
\end{tikzpicture}} &
\subcaptionbox{\miniboone\label{fig:foo}}{
\begin{tikzpicture}
\begin{axis}[xlabel={$\epsilon$}, ylabel= {average relative error},
    width = \imgwidth,
	height = \imgheight,
    cycle list name=yaf,
	tick scale binop=\times,
	scale only axis,
	no markers
    ]
\addplot+[yafcolor5] table[x index = 0, y index = 3, header = false] {results/miniboone.scores.res};
\pgfplotsextra{\yafdrawaxis{0}{0.9}{0}{0.009}}
\end{axis}
\end{tikzpicture}}
\\
\subcaptionbox{\hepmass\label{fig:foo2}}{
\begin{tikzpicture}
\begin{axis}[xlabel={$\epsilon$}, ylabel= {max. relative error},
    width = \imgwidth,
	height = \imgheight,
    cycle list name=yaf,
	tick scale binop=\times,
	yticklabel style={/pgf/number format/fixed,/pgf/number format/precision=2},
	scale only axis,
	no markers
    ]
\addplot+[yafcolor5] table[x index = 0, y index = 4, header = false] {results/hepmass.scores.res};
\pgfplotsextra{\yafdrawaxis{0}{0.9}{0}{0.17}}
\end{axis}
\end{tikzpicture}} &

\subcaptionbox{\tvads\label{fig:foo2}}{
\begin{tikzpicture}
\begin{axis}[xlabel={$\epsilon$}, ylabel= {max. relative error},
    width = \imgwidth,
	height = \imgheight,
    cycle list name=yaf,
	tick scale binop=\times,
	yticklabel style={/pgf/number format/fixed,/pgf/number format/precision=2},
	scale only axis,
	scaled y ticks = false,
	no markers
    ]
\addplot+[yafcolor5] table[x index = 0, y index = 4, header = false] {results/tv.scores.res};
\pgfplotsextra{\yafdrawaxis{0}{0.9}{0}{0.027}}
\end{axis}
\end{tikzpicture}} &

\subcaptionbox{\miniboone\label{fig:foo2}}{
\begin{tikzpicture}
\begin{axis}[xlabel={$\epsilon$}, ylabel= {max. relative error},
    width = \imgwidth,
	height = \imgheight,
    cycle list name=yaf,
	tick scale binop=\times,
	yticklabel style={/pgf/number format/fixed,/pgf/number format/precision=2},
	scale only axis,
	no markers
    ]
\addplot+[yafcolor5] table[x index = 0, y index = 4, header = false] {results/miniboone.scores.res};
\pgfplotsextra{\yafdrawaxis{0}{0.9}{0}{0.15}}
\end{axis}
\end{tikzpicture}}

\end{tabular}

\caption{Actual relative error as a function of $\epsilon$. Top row: average error, bottom row: maximum error. Proposition~\ref{prop:approx}
states that error cannot be larger than $\epsilon / 2$.}
\label{fig:relative}

\end{figure}
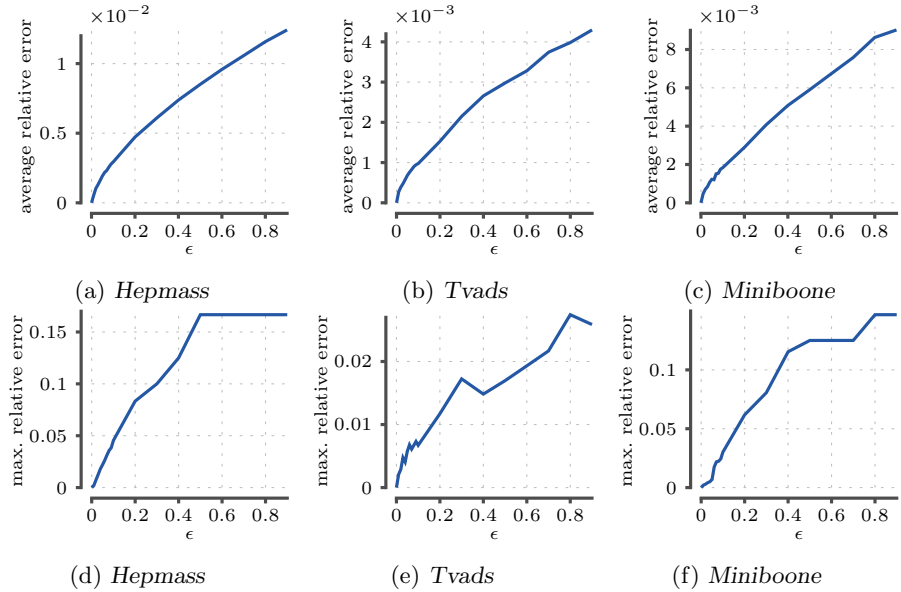

\textbf{Computational cost vs. error}:
Next, we test the trade-off between the computational cost and the relative error.
The top row of Figure~\ref{fig:tradeoff} shows the running time as a function
of the average error, while
the bottom row of Figure~\ref{fig:tradeoff} shows the size of $(1 + \epsilon)$-compressed list as a function
of the average error. Here, we used a window size of $1000$.

From the results, we see the trade-off between the error and the running time:
as the error increases, the running time drops. This is mainly due to the fewer
elements in the compressed list as demonstrated in the bottom row. The running
stabilizes for larger errors; this is due to the operations that do not depend
on $\epsilon$, such as maintaining binary tree $T$. 

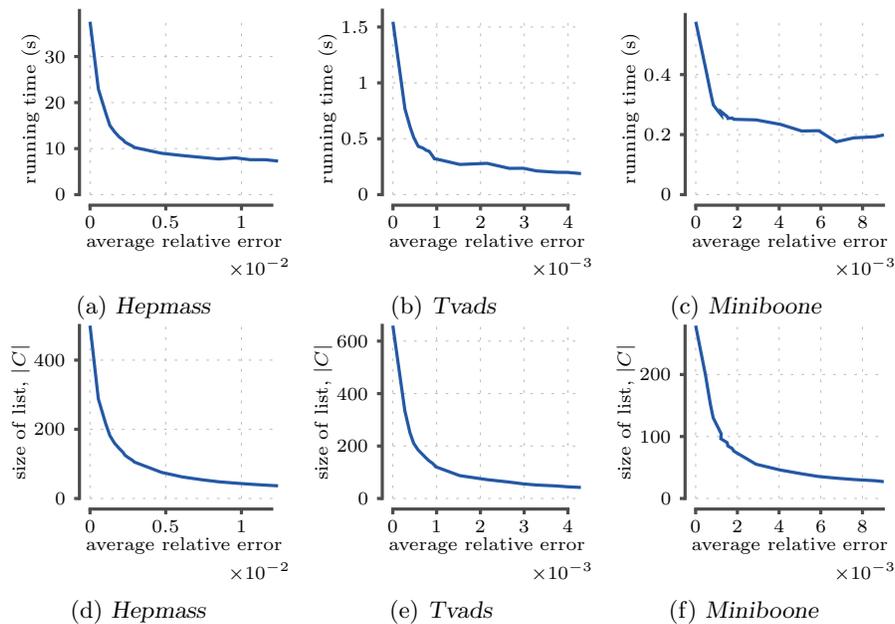
\begin{figure}[ht!]
\setlength{\imgwidth}{2.5cm}
\begin{tabular}{rrr}
\subcaptionbox{\hepmass\label{fig:foo2}}{
\begin{tikzpicture}
\begin{axis}[xlabel={average relative error}, ylabel= {running time (s)},
    width = \imgwidth,
	height = \imgheight,
    cycle list name=yaf,
	tick scale binop=\times,
	yticklabel style={/pgf/number format/fixed,/pgf/number format/precision=2},
	scale only axis,
	ymin = 0,
	no markers
    ]
\addplot+[yafcolor5] table[x index = 3, y expr = {\thisrowno{5} - 5.635}, header = false] {results/hepmass.scores.res};
\pgfplotsextra{\yafdrawaxis{0}{0.012}{0}{40}}
\end{axis}
\end{tikzpicture}} &

\subcaptionbox{\tvads\label{fig:foo2}}{
\begin{tikzpicture}
\begin{axis}[xlabel={average relative error}, ylabel= {running time (s)},
    width = \imgwidth,
	height = \imgheight,
    cycle list name=yaf,
	tick scale binop=\times,
	yticklabel style={/pgf/number format/fixed,/pgf/number format/precision=2},
	scale only axis,
	ymin = 0,
	no markers
    ]
\addplot+[yafcolor5] table[x index = 3, y expr = {\thisrowno{5} - 0.158}, header = false] {results/tv.scores.res};
\pgfplotsextra{\yafdrawaxis{0}{0.004}{0}{1.65}}
\end{axis}
\end{tikzpicture}} &
\subcaptionbox{\miniboone\label{fig:foo2}}{
\begin{tikzpicture}
\begin{axis}[xlabel={average relative error}, ylabel= {running time (s)},
    width = \imgwidth,
	height = \imgheight,
    cycle list name=yaf,
	tick scale binop=\times,
	yticklabel style={/pgf/number format/fixed,/pgf/number format/precision=2},
	scale only axis,
	ymin = 0,
	no markers
    ]
\addplot+[yafcolor5] table[x index = 3, y expr = {\thisrowno{5} - 0.171}, header = false] {results/miniboone.scores.res};
\pgfplotsextra{\yafdrawaxis{0}{0.009}{0}{0.6}}
\end{axis}
\end{tikzpicture}}

\\
\subcaptionbox{\hepmass\label{fig:foo2}}{
\begin{tikzpicture}
\begin{axis}[xlabel={average relative error}, ylabel= {size of list, $\abs{C}$},
    width = \imgwidth,
	height = \imgheight,
    cycle list name=yaf,
	tick scale binop=\times,
	scale only axis,
	ymin = 0,
	no markers
    ]
\addplot+[yafcolor5] table[x index = 3, y index = 1, header = false] {results/hepmass.scores.res};
\pgfplotsextra{\yafdrawaxis{0}{0.012}{0}{500}}
\end{axis}
\end{tikzpicture}} &
\subcaptionbox{\tvads\label{fig:foo2}}{
\begin{tikzpicture}
\begin{axis}[xlabel={average relative error}, ylabel= {size of list, $\abs{C}$},
    width = \imgwidth,
	height = \imgheight,
    cycle list name=yaf,
	tick scale binop=\times,
	scale only axis,
	ymin = 0,
	no markers
    ]
\addplot+[yafcolor5] table[x index = 3, y index = 1, header = false] {results/tv.scores.res};
\pgfplotsextra{\yafdrawaxis{0}{0.004}{0}{660}}
\end{axis}
\end{tikzpicture}} &
\subcaptionbox{\miniboone\label{fig:foo2}}{
\begin{tikzpicture}
\begin{axis}[xlabel={average relative error}, ylabel= {size of list, $\abs{C}$},
    width = \imgwidth,
	height = \imgheight,
    cycle list name=yaf,
	tick scale binop=\times,
	scale only axis,
	ymin = 0,
	no markers
    ]
\addplot+[yafcolor5] table[x index = 3, y index = 1, header = false] {results/miniboone.scores.res};
\pgfplotsextra{\yafdrawaxis{0}{0.009}{0}{278}}
\end{axis}
\end{tikzpicture}}
\end{tabular}

\caption{Top row: running time as a function of average relative error.
Bottom row: size of the compressed list $\abs{C}$ as a function of average relative error.
}
\label{fig:tradeoff}

\end{figure}
{
\setlength{\textfloatsep}{0pt}
\begin{figure}[ht!]
\floatbox[{\capbeside\thisfloatsetup{capbesideposition={left,top},capbesidewidth=5cm}}]{figure}
{\caption{A speed-up of estimating AUC with $\epsilon = 0.1$
against computing AUC exactly, as a function of sliding window size.
The dataset is \miniboone.}\label{fig:wintime}}
{
\begin{tikzpicture}
\begin{axis}[xlabel={window size}, ylabel= {speedup},
    width = 5cm,
	height = \imgheight,
    cycle list name=yaf,
	tick scale binop=\times,
	yticklabel style={/pgf/number format/fixed,/pgf/number format/precision=2},
	scale only axis,
	x tick label style = {/pgf/number format/set thousands separator = {\,}},
	scaled x ticks = false,
	no markers
    ]
\addplot+[yafcolor5] table[x index = 0, y expr = {(\thisrowno{2} - 0.171 ) / (\thisrowno{1} - 0.171 )}, header = false] {results/miniboone.win.res};
\pgfplotsextra{\yafdrawaxis{1000}{10000}{2}{17}}
\end{axis}
\end{tikzpicture}}
\end{figure}
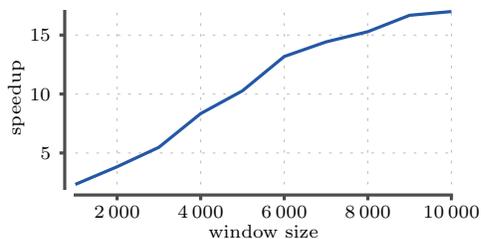}

\textbf{Computational cost vs. window size}:
Computing exact AUC requires $\bigO{k}$ time while estimating AUC is $\bigO{\log k / \epsilon}$.
Consequently, the speed-up should increase as the size of the sliding window increases.
We demonstrate this effect in Figure~\ref{fig:wintime} using the \miniboone dataset.
We see that the speed-up increases as a function of window size: computing estimates using $\epsilon = 0.1$
is 17 times faster for a window size of $10\,000$.

\section{Concluding remarks}\label{sec:conclusions}
In this paper we introduced an approximation scheme that allows to maintain an
estimate AUC in a sliding window within the guaranteed relative error of
$\epsilon / 2$ in $\bigO{(\log k) / \epsilon}$ time.  The key idea behind the
estimator
is to group the data points. The grouping has to be done cleverly so that the
error stays small, the number of groups stay small, and the list can be updated
quickly. We achieve this by maintaining groups, where the number of positive labels
can only increase relatively by $(1 + \epsilon)$ within one group, and must increase
by at least $(1 + \epsilon)$ within two groups.
Our experimental evaluation suggests that the average error in practice is much
smaller than the guaranteed approximation, and that we can achieve significant
speed-up, especially as the window size grows.

Our algorithm relies on the fact that the data points have no weights,
speci\-fically, Lemma~\ref{lem:addcorrpos} relies on the fact that the update may
change the counters only by 1. If the data points are weighted, a different
approach is required: It is possible to construct $(1 + \epsilon)$-list from a
scratch. The key idea here is a new query, where, given a threshold $\sigma$,
we look for a node $v$ that has the largest $\headpos{v}$ such that
$\headpos{v} \leq \sigma$. This query can be done using the same trick as
in \cumstats, and it requires $\bigO{\log k}$ time. The list can be then
constructed by calling this query with exponentially increasing thresholds $\bigO{(\log k) / \epsilon}$ times.
This leads to a running time of $\bigO{(\log^2 k) / \epsilon}$. An interesting direction
for future work is to improve this complexity to, say, $\bigO{(\log k) / \epsilon}$.

\bibliography{references}

\end{document}